\documentclass[sigconf,nonacm]{acmart}
\usepackage{amsmath}

\usepackage{amssymb}
\usepackage{amsthm}
\usepackage{mathtools}
\usepackage{graphicx}
\usepackage{booktabs}
\usepackage{xcolor}
\usepackage{enumitem}

\newtheorem{theorem}{Theorem}
\newtheorem{corollary}[theorem]{Corollary}
\newtheorem{definition}[theorem]{Definition}

\acmYear{2026}
\copyrightyear{2026}
\acmISBN{}
\acmDOI{}

\title{Kolmogorov Complexity Bounds\\for LLM Steganography and a Perplexity-Based\\Detection Proxy}

\author{Andrii Shportko}
\affiliation{%
  \institution{Northwestern University}
  \city{Evanston}
  \country{IL}
}

\begin{abstract}
Large language models can rewrite text to embed hidden payloads while preserving surface-level meaning, a capability that opens covert channels between cooperating AI systems and poses challenges for alignment monitoring.
We study the information-theoretic cost of such embedding.
Our main result is that any steganographic scheme that preserves the semantic load of a covertext~$M_1$ while encoding a payload~$P$ into a stegotext~$M_2$ must satisfy
$K(M_2) \geq K(M_1) + K(P) - O(\log n)$,
where $K$ denotes Kolmogorov complexity and $n$ is the combined message length.
A corollary is that any non-trivial payload forces a strict complexity increase in the stegotext, regardless of how cleverly the encoder distributes the signal.

Because Kolmogorov complexity is uncomputable, we ask whether practical proxies can detect this predicted increase.
Drawing on the classical correspondence between lossless compression and Kolmogorov complexity, we argue that language-model perplexity occupies an analogous role in the probabilistic regime and propose the Binoculars perplexity-ratio score as one such proxy.
Preliminary experiments with a color-based LLM steganographic scheme support the theoretical prediction: a paired $t$-test over 300 samples yields $t = 5.11$, $p < 10^{-6}$.
\end{abstract}



\begin{document}

\maketitle

\section{Introduction}
\label{sec:intro}

Large language models can rewrite text to embed covert signals while preserving surface-level meaning.
This capability creates steganographic channels between cooperating AI systems---channels that may evade alignment monitors and undermine oversight of chain-of-thought reasoning~\cite{Roger2023}.
As models grow more capable, the bandwidth and subtlety of such channels are likely to increase, and understanding their fundamental limits becomes a prerequisite for principled defense.

We study two questions.
First, does semantic-preserving steganographic embedding necessarily increase the intrinsic descriptive complexity of the output?
Perfectly secure stegosystems with zero KL divergence between covertext and stegotext distributions exist for certain source classes~\cite{Cachin2004,Ryabko2008}, yet the relationship between such statistical security and the Kolmogorov complexity of the steganographic output has received little formal attention.
Second, given that Kolmogorov complexity is uncomputable, can practical metrics detect the predicted complexity increase?

Our contributions are as follows.
\begin{enumerate}[leftmargin=*,nosep]
\item We prove that semantic-preserving steganographic embedding increases Kolmogorov complexity by at least $K(P) - O(\log n)$, where $P$ is the payload and $n$ is the message length (Section~\ref{sec:theory}).
\item We identify language-model perplexity as a computable proxy for this complexity increase, grounded in the classical correspondence between compression algorithms and Kolmogorov complexity, and we propose the Binoculars score~\cite{Hans2024} as a concrete instantiation (Section~\ref{sec:proxy}).
\item We validate the theoretical prediction on a 300-sample LLM steganographic scheme, confirming that the Binoculars score significantly distinguishes stegotext from baselines ($p < 10^{-6}$; Section~\ref{sec:experiments}).
\end{enumerate}

\section{Related Work}
\label{sec:related}

\paragraph{Information-theoretic steganography.}
Cachin~\cite{Cachin2004} formalized steganographic security via relative entropy.
Ryabko and Ryabko~\cite{Ryabko2009} showed that for sources with maximal Kolmogorov complexity, any high-speed perfect stegosystem must have exponential complexity in its system parameters---a result about \emph{encoder} complexity.
Our Theorem~\ref{thm:main} complements theirs by bounding the \emph{output} complexity.
Muchnik~\cite{Muchnik2011} studied Kolmogorov complexity in cryptographic contexts but without a semantic-preservation constraint.

\paragraph{LLM covert channels.}
Roger and Greenblatt~\cite{Roger2023} studied AI-to-AI covert communication and introduced capacity formulas for noisy steganographic channels.
Shportko~\cite{Shportko2025} demonstrated empirically that linguistic steganography, just like adversarial paraphrasing ~\cite{shportko-2025-paraphrasing}, significantly increases Binoculars scores, providing independent evidence for the complexity-increase phenomenon we formalize here.

\paragraph{Compression as a Kolmogorov complexity proxy.}
Cilibrasi and Vit\'{a}nyi~\cite{Cilibrasi2005} introduced the Normalized Compression Distance, treating the output length $|C(x)|$ of a compressor $C$ as an operational stand-in for $K(x)$.
Benedetto et al.~\cite{Benedetto2002} used gzip-based complexity estimates for language classification.
We extend this reasoning to language-model perplexity, treating it as a probabilistic analogue of compression length.

\section{Theoretical Results}
\label{sec:theory}

\paragraph{Notation.}
Throughout, $K(x)$ denotes the Kolmogorov complexity of a finite binary string~$x$: the length of the shortest program that produces~$x$ on a fixed universal Turing machine~\cite{Li1997}.
$K(x \mid y)$ denotes conditional Kolmogorov complexity, where $y$ is provided as auxiliary input.
$|x|$ denotes the bit-length of~$x$.
All $O(\cdot)$ terms hide constants that depend on the choice of universal machine but not on the messages.

\begin{definition}[Identical Semantic Load]
\label{def:semantic}
Two messages $M_1$ and $M_2$ have \emph{identical semantic load} if there exists a finite binary string $S$ (the ``semantic description'') such that
\[
K(M_1 \mid S) = O(1) \quad\text{and}\quad K(M_2 \mid S) = O(1).
\]
That is, given $S$, each message can be reconstructed with only a constant number of additional bits.
\end{definition}

\begin{definition}[Lossless Semantic-Preserving Embedding]
\label{def:embed}
A function
\[
\mathrm{Embed}\colon \mathcal{M} \times \mathcal{P} \;\to\; \mathcal{M}
\]
is a \emph{lossless semantic-preserving embedding} if two conditions hold:
\begin{enumerate}[label=(\alph*),nosep]
\item \emph{Losslessness}: $ \exists \mathrm{Extract}\colon\mathcal{M}\to\mathcal{P}$: $\mathrm{Extract}(\mathrm{Embed}(M_1,P))$ $= P$ for all $M_1 \in \mathcal{M}$ and $P \in \mathcal{P}$.
\item \emph{Semantic preservation}: $M_1$ and $M_2 = \mathrm{Embed}(M_1,P)$ have identical semantic load.
\end{enumerate}
Here $\mathcal{M}$ is the message space, $\mathcal{P}$ is the payload space, $M_1$ is the covertext, $P$ is the payload, and $M_2$ is the stegotext.
\end{definition}

\begin{theorem}
\label{thm:main}
Let $M_2 = \mathrm{Embed}(M_1, P)$ for a lossless semantic-preserving embedding.
Let $n = |M_1| + |M_2|$.
Then
\[
K(M_2) \;\geq\; K(M_1) + K(P) - O(\log n).
\]
\end{theorem}

\begin{proof}
Let $S$ denote the semantic description with $K(M_i \mid S) = O(1)$ for $i \in \{1,2\}$.

\emph{Step 1 (From $M_2$ we can recover $M_1$ and $P$).}
Since $\mathrm{Extract}$ is a fixed-length program, we have
\[
K(P \mid M_2) \;\leq\; |\langle\mathrm{Extract}\rangle| + O(1) \;=\; O(1).
\]
Since $K(M_2 \mid S) = O(1)$ by semantic preservation, $S$ is recoverable from $M_2$ with constant overhead; formally $K(S \mid M_2) = O(1)$.
From $S$ we recover $M_1$ because $K(M_1 \mid S) = O(1)$.
Chaining these three recoveries:
\[
K(M_1, P \mid M_2) \;=\; O(1).
\]

\emph{Step 2 (From $M_1$ and $P$ we can recover $M_2$).}
Since $\mathrm{Embed}$ is a fixed function:
\[
K(M_2 \mid M_1, P) \;\leq\; |\langle\mathrm{Embed}\rangle| + O(1) \;=\; O(1).
\]

\emph{Step 3 (Applying the symmetry of information).}
Steps~1 and~2 establish that $M_2$ and the pair $(M_1,P)$ contain the same information up to $O(1)$ bits.
By the symmetry-of-information theorem~\cite{Li1997}:
\[
K(M_2) \;=\; K(M_1, P) + O(\log n).
\]

\emph{Step 4 (Expanding the joint complexity).}
By the chain rule for Kolmogorov complexity:
\[
K(M_1, P) \;=\; K(M_1) + K(P \mid M_1) + O(\log n).
\]

\emph{Step 5 (Payload independence).}
The payload $P$ is chosen independently of the covertext $M_1$: by construction, $P$ carries information orthogonal to the semantic description $S$ of $M_1$.
Therefore $M_1$ provides no information about $P$ beyond what is already available, and we have
\[
K(P \mid M_1) \;\geq\; K(P) - O(1).
\]

\emph{Step 6 (Combining).}
Substituting Steps~4 and~5 into Step~3:
\[
K(M_2) \;\geq\; K(M_1) + K(P) - O(\log n). \qedhere
\]
\end{proof}

\begin{corollary}[Strict Complexity Increase]
\label{cor:strict}
If the payload has non-trivial complexity such that $K(P) \gg \log n$, then
\[
K(M_2) \;>\; K(M_1).
\]
That is, any non-trivial payload forces a strict complexity increase in the stegotext.
\end{corollary}

\section{From Kolmogorov Complexity to Computable Proxies}
\label{sec:proxy}

Theorem~\ref{thm:main} predicts a complexity increase, but Kolmogorov complexity is uncomputable~\cite{Li1997}.
We now argue that language-model perplexity serves as a practical proxy, grounding this argument in the classical correspondence between compression and Kolmogorov complexity.

\subsection{Compression as a Complexity Proxy}

For any lossless compressor $C$ and string $x$:
\[
K(x) \;\leq\; |C(x)| + c_C,
\]
where $c_C$ is a constant depending only on the compressor.
This inequality---a direct consequence of the invariance theorem~\cite{Li1997}---justifies treating $|C(x)|$ as an upper bound on $K(x)$.
Cilibrasi and Vit\'{a}nyi~\cite{Cilibrasi2005} exploited this to define the Normalized Compression Distance; Benedetto et al.~\cite{Benedetto2002} used gzip output lengths for language classification.
In both cases, changes in $|C(x)|$ serve as a computable signal for changes in $K(x)$.

\subsection{Language-Model Cross-Entropy as a Probabilistic Compressor}

A language model $\mathcal{M}$ assigns a probability $P_{\mathcal{M}}(x_i \mid x_{<i})$ to each token $x_i$ given its prefix.
The cross-entropy of $\mathcal{M}$ on a string $x = x_1 \cdots x_T$ is
\[
H_{\mathcal{M}}(x) \;=\; -\frac{1}{T}\sum_{i=1}^{T} \log_2 P_{\mathcal{M}}(x_i \mid x_{<i}).
\]
This quantity upper-bounds the per-symbol entropy rate $h(X)$ of the source~\cite{Shannon1951,Cover2006}.
Because the entropy rate lower-bounds the asymptotic per-symbol Kolmogorov complexity for ergodic sources, we obtain:
\[
h(X) \;\leq\; H_{\mathcal{M}}(x) \;\leq\; \frac{|C(x)|}{T} + O(T^{-1}).
\]
Language-model perplexity, defined as
\[
\mathrm{PPL}_{\mathcal{M}}(x) \;=\; 2^{H_{\mathcal{M}}(x)},
\]
is therefore an exponential-scale proxy for per-symbol algorithmic complexity, in the same way that $|C(x)|$ is a linear-scale proxy.
If Theorem~\ref{thm:main} predicts that $K(M_2) > K(M_1)$, then a sufficiently capable model should assign higher cross-entropy to $M_2$ than to $M_1$, because the additional $K(P)$ bits must manifest as atypical token choices that increase model surprisal.

\subsection{The Binoculars Score}

The Binoculars method~\cite{Hans2024} defines a score for a string $s$ as
\[
B_{\mathcal{M}_1,\mathcal{M}_2}(s) \;=\; \frac{\log \mathrm{PPL}_{\mathcal{M}_1}(s)}{\log \mathrm{X\text{-}PPL}_{\mathcal{M}_1,\mathcal{M}_2}(s)},
\]
where $\mathcal{M}_1$ is a base language model (Falcon-7B), $\mathcal{M}_2$ is its instruction-tuned variant (Falcon-7B-Instruct), and $\mathrm{X\text{-}PPL}$ denotes cross-perplexity.
Originally designed for AI-generated text detection, we repurpose it for steganalysis: the ratio captures how much more surprising a string is to the base model relative to the instruction-tuned model.
Steganographic encoding, which selects atypical tokens to embed payload, should elevate this ratio.

We term the resulting detection strategy \emph{Quadroculars}: using the Binoculars score \emph{difference} between stegotext and a baseline as a steganalysis signal.

\subsection{Quadroculars as a Complexity Difference Proxy}

We now formalize the sense in which the Quadroculars score---the
difference in Binoculars scores between stegotext and covertext---tracks
the Kolmogorov complexity difference predicted by Theorem~\ref{thm:main}.

\begin{definition}[Quadroculars Score]
\label{def:quadroculars}
For a base model $\mathcal{M}_1$ and a reference model $\mathcal{M}_2$,
the \emph{Quadroculars score} of a stegotext--covertext pair $(M_2, M_1)$ is
\[
Q(M_2, M_1)
\;=\;
B_{\mathcal{M}_1,\mathcal{M}_2}(M_2)
\;-\;
B_{\mathcal{M}_1,\mathcal{M}_2}(M_1),
\]
where $B_{\mathcal{M}_1,\mathcal{M}_2}(\cdot)$ is the Binoculars score.
\end{definition}

We require three assumptions, each of which admits a precise error term.

\begin{description}[leftmargin=*,nosep]
\item[A1 (Model quality).]
The base language model $\mathcal{M}_1$ is
$\varepsilon$-optimal for the source in the sense that for every
string $x$ of length $T$,
\[
\left|\; H_{\mathcal{M}_1}(x) \;-\; \frac{K(x)}{T} \;\right|
\;\leq\; \varepsilon,
\]
where $H_{\mathcal{M}_1}(x)$ is the per-token cross-entropy and
$K(x)/T$ is the per-symbol Kolmogorov complexity rate.

\item[A2 (Cross-perplexity stability).]
The cross-perplexity denominator in the Binoculars score is
approximately constant across semantically equivalent messages.
Concretely, there exists $\delta \geq 0$ such that
\[
\bigl|\,\log\mathrm{X\text{-}PPL}_{\mathcal{M}_1,\mathcal{M}_2}(M_2)
\;-\;
\log\mathrm{X\text{-}PPL}_{\mathcal{M}_1,\mathcal{M}_2}(M_1)\,\bigr|
\;\leq\; \delta.
\]

\item[A3 (Equal length).]
The covertext and stegotext have equal token length:
$|M_1| = |M_2| = T$.
\end{description}

\begin{proposition}
\label{prop:quadroculars}
Under assumptions \textbf{A1}--\textbf{A3}, the Quadroculars score
satisfies
\[
Q(M_2, M_1)
\;=\;
\frac{K(M_2) - K(M_1)}{T \cdot \log\mathrm{X\text{-}PPL}_{\mathcal{M}_1,\mathcal{M}_2}(M_1)}
\;+\; R,
\]
where the residual $R$ is bounded by
\[
|R|
\;\leq\;
\frac{2\varepsilon}
     {\log\mathrm{X\text{-}PPL}_{\mathcal{M}_1,\mathcal{M}_2}(M_1)
      - \delta}
\;+\; O(\delta).
\]
In particular, if $\varepsilon \to 0$ and $\delta \to 0$
(i.e., the language model becomes a perfect compressor and the
cross-perplexity is stable), then
\[
Q(M_2, M_1)
\;\propto\;
K(M_2) - K(M_1).
\]
\end{proposition}

\begin{proof}
Write $D = \log\mathrm{X\text{-}PPL}_{\mathcal{M}_1,\mathcal{M}_2}(M_1)$
for the cross-perplexity denominator evaluated on the covertext.
By definition of the Binoculars score:
\[
B_{\mathcal{M}_1,\mathcal{M}_2}(x)
\;=\;
\frac{\log\mathrm{PPL}_{\mathcal{M}_1}(x)}
     {\log\mathrm{X\text{-}PPL}_{\mathcal{M}_1,\mathcal{M}_2}(x)}
\;=\;
\frac{H_{\mathcal{M}_1}(x)}
     {\log\mathrm{X\text{-}PPL}_{\mathcal{M}_1,\mathcal{M}_2}(x)},
\]
since $\log\mathrm{PPL}_{\mathcal{M}_1}(x) = H_{\mathcal{M}_1}(x)$.

\emph{Step 1 (Replace cross-entropy with Kolmogorov complexity).}
By assumption \textbf{A1}:
\[
H_{\mathcal{M}_1}(x) \;=\; \frac{K(x)}{T} + \eta_x,
\quad\text{where } |\eta_x| \leq \varepsilon.
\]

\emph{Step 2 (Stabilize the denominator).}
By assumption \textbf{A2}, the denominator for $M_2$ satisfies
\[
\log\mathrm{X\text{-}PPL}_{\mathcal{M}_1,\mathcal{M}_2}(M_2)
\;=\; D + \xi,
\quad\text{where } |\xi| \leq \delta.
\]

\emph{Step 3 (Expand the Quadroculars score).}
Substituting into the definition:
\begin{align*}
Q(M_2, M_1)
&= \frac{K(M_2)/T + \eta_{M_2}}{D + \xi}
   \;-\;
   \frac{K(M_1)/T + \eta_{M_1}}{D} \\[6pt]
&= \frac{1}{D}\!\left(\frac{K(M_2) - K(M_1)}{T}\right)
   + \frac{\eta_{M_2} - \eta_{M_1}}{D}
   + E_\xi,
\end{align*}
where $E_\xi$ collects terms arising from the denominator perturbation $\xi$.

\emph{Step 4 (Bound the residual).}
The second term has magnitude at most $2\varepsilon / D$.
For the denominator perturbation, a first-order expansion gives
\[
|E_\xi|
\;=\;
\left|\frac{K(M_2)/T + \eta_{M_2}}{D + \xi}
      - \frac{K(M_2)/T + \eta_{M_2}}{D}\right|
\;\leq\;
\frac{(K(M_2)/T + \varepsilon)\,|\xi|}{D(D-\delta)}.
\]
Since $K(M_2)/T$ is bounded (it cannot exceed the token-level log of the
alphabet size), this is $O(\delta)$.
Combining:
\[
Q(M_2, M_1)
\;=\;
\frac{K(M_2) - K(M_1)}{T \cdot D}
\;+\; R,
\]
with $|R| \leq 2\varepsilon/(D - \delta) + O(\delta)$.
\end{proof}

\paragraph{Interpretation.}
Proposition~\ref{prop:quadroculars} decomposes the Quadroculars score
into a \emph{signal term} proportional to the Kolmogorov complexity
difference $K(M_2) - K(M_1)$ and a \emph{residual} controlled by two
quantities: the model approximation error~$\varepsilon$ and the
cross-perplexity stability~$\delta$.

The signal term is normalized by $T \cdot D$, where $T$ is the token
length and $D$ is the cross-perplexity of the covertext.  This
normalization is natural: longer messages dilute the per-token complexity
difference, and higher baseline cross-perplexity reduces the
signal-to-noise ratio of the ratio-based Binoculars score.

Assumption~\textbf{A1} is the strongest requirement.  No finite language
model achieves $\varepsilon = 0$, but the assumption becomes
increasingly realistic as model capacity grows: a model that assigns
near-optimal log-loss on a text distribution is, in Shannon's sense, a
near-optimal compressor of that distribution~\cite{Shannon1951}.
Assumption~\textbf{A2} is empirically plausible when $M_1$ and $M_2$
have identical semantic load, since the instruction-tuned model
$\mathcal{M}_2$ (which drives the cross-perplexity denominator) should
respond similarly to two texts that convey the same meaning.

Combining with Theorem~\ref{thm:main}: the signal term satisfies
\[
\frac{K(M_2) - K(M_1)}{T \cdot D}
\;\geq\;
\frac{K(P) - O(\log n)}{T \cdot D},
\]
so the Quadroculars score is bounded below by a quantity proportional to
the payload complexity per token, divided by the baseline
cross-perplexity.  This provides a theoretical floor on the
detectability of steganographic embedding via the Binoculars score.
\section{Experimental Validation}
\label{sec:experiments}

We validate the theoretical prediction using a concrete LLM steganographic scheme.

\subsection{Synesthetic Encoding}

The scheme embeds 3-bit color codes per sentence.
The binary payload $P$ is segmented into 3-bit chunks, each mapped to one of 8 colors (\texttt{red}, \texttt{yellow}, \texttt{green}, \texttt{blue}, \texttt{purple}, \texttt{brown}, \texttt{black}, \texttt{white}), yielding
\[
\log_2 8 = 3 \text{ bits per sentence.}
\]
An LLM rewrites each covertext sentence so that it evokes the assigned color through stereotypically associated nouns, without naming the color explicitly.
A separate LLM decodes each stegotext sentence by classifying its dominant color association.

\paragraph{Example.}
Consider the covertext sentence ``The project was finally completed after months of effort'' with assigned color \textcolor{red}{red}.
The encoder rewrites it as ``The project was finally completed after months of \textcolor{red}{bloodshed} and effort,'' embedding a red-associated noun without naming the color.  The decoder, presented with the stegotext in isolation, identifies \textcolor{red}{bloodshed} as the dominant chromatic cue and outputs \texttt{red}, recovering the 3-bit payload $P = 000$.

\subsection{Dataset and Accuracy}

We construct 300 entries by independently sampling $N \in \{1,\ldots,6\}$ sentences from a corpus of \emph{Three-Body Problem} excerpts and $N$ colors uniformly from the palette, then pairing them randomly.
All encoding uses \texttt{claude-sonnet-4-20250514}.

Direct decoding achieves $95.7\%$ per-sentence accuracy (Figure~\ref{fig:color_accuracy}).
An adversarial LLM-based paraphraser reduces accuracy to $54.8\%$---still $4.4\times$ above the $12.5\%$ chance baseline.
Figure~\ref{fig:confusion} shows the full confusion matrices; paraphrasing introduces a systematic bias toward \texttt{black}.

Following Roger and Greenblatt~\cite{Roger2023}, effective channel capacity per sentence is
\[
C_{\mathrm{eff}} \;=\; n_b - f\log_2\!\left(\frac{1}{f}\right) - (1-f)\log_2\!\left(\frac{2^{n_b}-1}{1-f}\right),
\]
where $n_b$ is attempted number of bits encoded per kB and $f$ is the per-sentence perfect-decoding frequency.
This yields approximately $39\;\text{bits/kB}$ for encoded text and ${\sim}\,14\;\text{bits/kB}$ after paraphrasing.

\begin{figure}[t]
  \centering
  \includegraphics[width=\columnwidth]{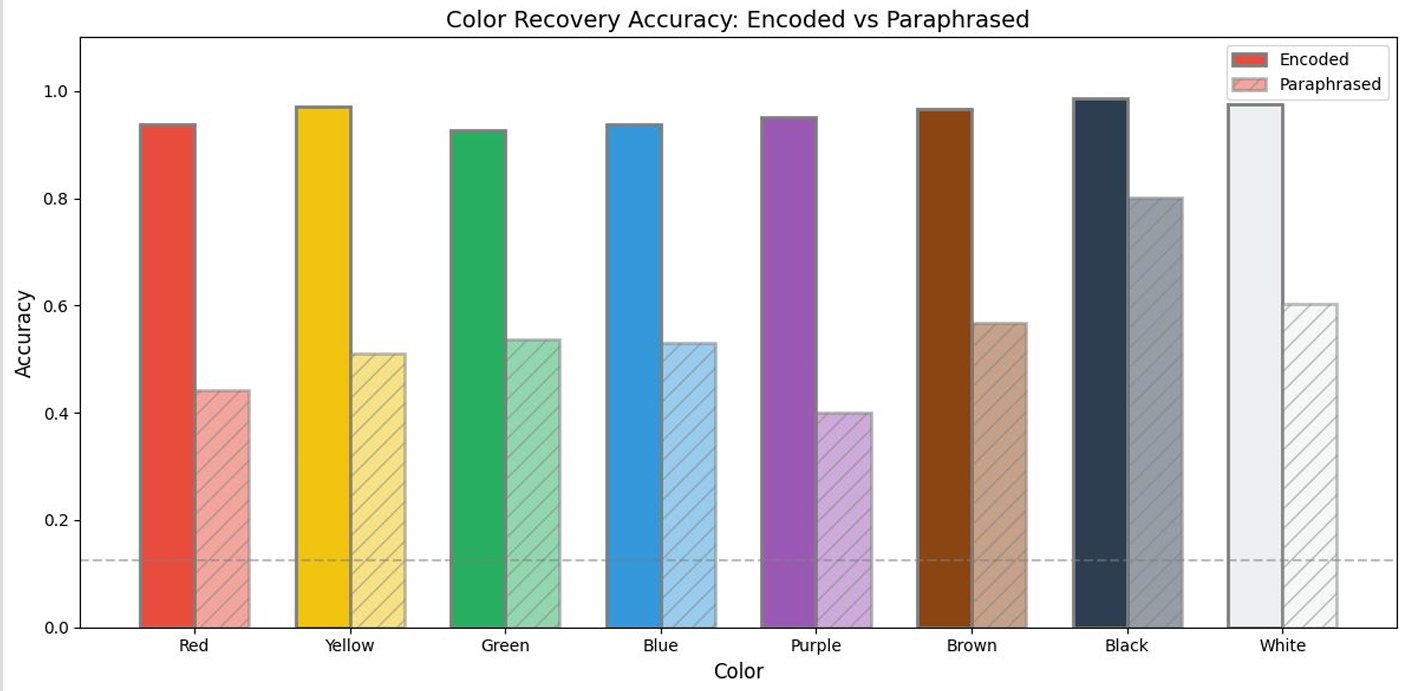}
  \caption{Per-color decoding accuracy for encoded (solid) vs.\ paraphrased (hatched) text. Dashed line: chance level ($12.5\%$).}
  \label{fig:color_accuracy}
\end{figure}

\begin{figure*}[t]
  \centering
  \includegraphics[width=\textwidth]{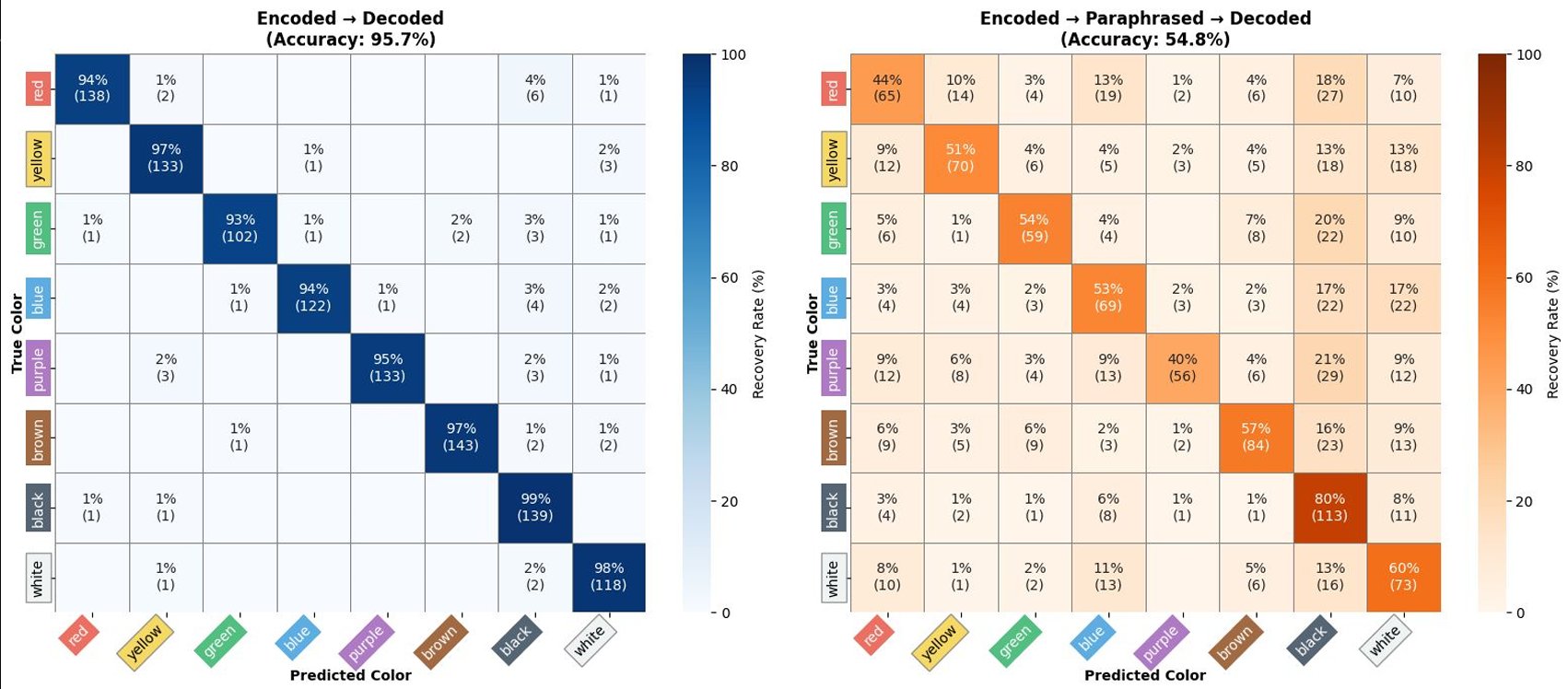}
  \caption{Confusion matrices. Left: direct decoding ($95.7\%$). Right: after adversarial paraphrasing ($54.8\%$). Paraphrasing introduces a strong bias toward \texttt{black}.}
  \label{fig:confusion}
\end{figure*}

\subsection{Binoculars Score Validation}

We compute Binoculars scores for all 300 encoded entries and their paraphrased counterparts (Table~\ref{tab:binoculars}).
Encoded text has mean score $1.649$ (std $1.230$); paraphrased text scores $1.338$ (std $0.728$).
A paired $t$-test yields
\[
t = 5.11, \quad p = 5.72 \times 10^{-7},
\]
and the shift direction is consistent with the theoretical prediction: encoding increases complexity, and paraphrasing (which strips the payload) decreases it.
Figure~\ref{fig:binoculars} shows the score distributions; Figure~\ref{fig:paired_diff} shows the paired differences, with mean $-0.310$ and $63.3\%$ of pairs exhibiting a decrease after paraphrasing.

These results converge with independent findings that acrostic ciphers also significantly increase Binoculars scores (Cohen's $d = -2.23$, $p < 0.0001$)~\cite{Shportko2025}.

\begin{table}[t]
\centering
\caption{Binoculars score statistics.}
\label{tab:binoculars}
\small
\begin{tabular}{@{}lcccc@{}}
\toprule
\textbf{Condition} & \textbf{Mean} & \textbf{Std} & \textbf{Min} & \textbf{Max} \\
\midrule
Encoded     & 1.649 & 1.230 & 0.443 & 14.649 \\
Paraphrased & 1.338 & 0.728 & 0.380 & 5.595 \\
\midrule
\multicolumn{5}{@{}l}{\small Paired $t$: $t\!=\!5.11$, $p\!=\!5.72\!\times\!10^{-7}$; mean shift $-0.310$; $63.3\%$ pairs decrease.} \\
\bottomrule
\end{tabular}
\end{table}

\begin{figure}[t]
  \centering
  \includegraphics[width=\columnwidth]{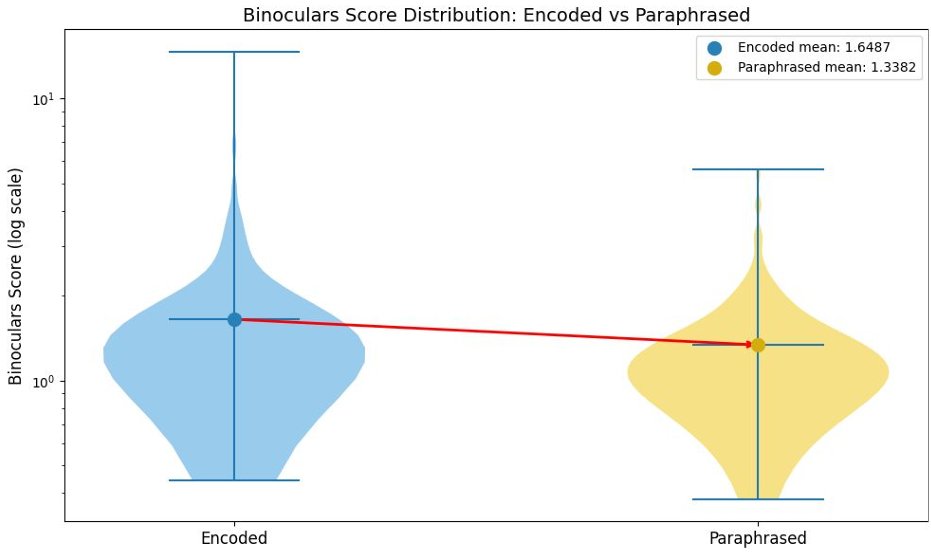}
  \caption{Binoculars score distributions for encoded vs.\ paraphrased text. Encoding significantly increases scores, consistent with Theorem~\ref{thm:main}.}
  \label{fig:binoculars}
\end{figure}

\begin{figure}[t]
  \centering
  \includegraphics[width=\columnwidth]{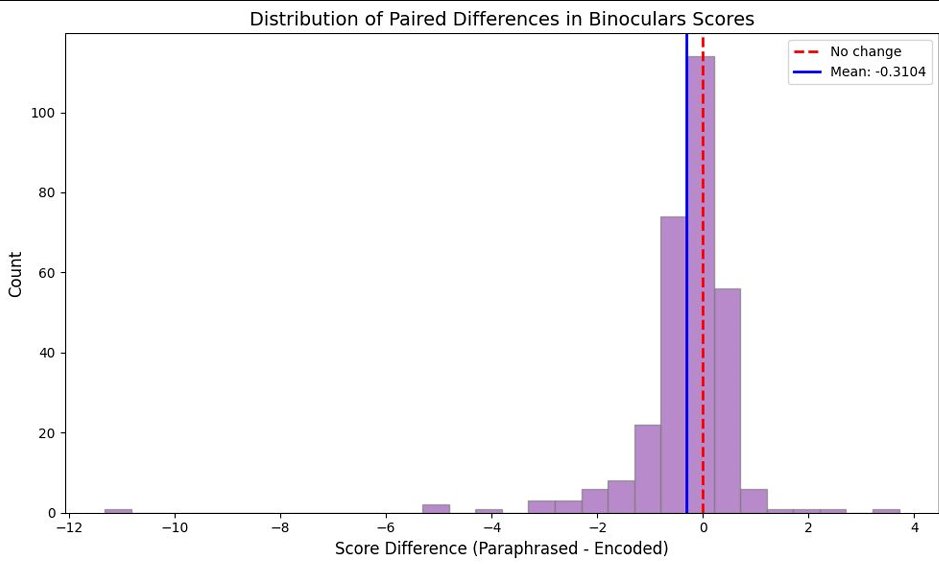}
  \caption{Paired Binoculars score differences (paraphrased $-$ encoded). The negative mean ($-0.310$) confirms that encoding increases complexity.}
  \label{fig:paired_diff}
\end{figure}

\section{Discussion and Conclusion}
\label{sec:conclusion}

Theorem~\ref{thm:main} establishes a ``no free lunch'' principle for LLM steganography: semantic-preserving embedding must increase Kolmogorov complexity by at least the payload complexity.
Combined with Ryabko and Ryabko's~\cite{Ryabko2009} bound on encoder complexity, this implies that steganography incurs fundamental costs at both the system and the output levels.
The result applies to any semantic-preserving embedding---including watermarking schemes that modify LLM output distributions at generation time---and places an information-theoretic floor on the detectability of such modifications.

Our proxy argument rests on a chain of established correspondences: compressors approximate $K(x)$ from above~\cite{Li1997,Cilibrasi2005}, language-model cross-entropy approximates the entropy rate from above~\cite{Shannon1951,Cover2006}, and the Binoculars score captures relative surprisal between two models.
The empirical success of this proxy ($p < 10^{-6}$) provides evidence that the chain is tight enough to detect the $K(P)$ signal predicted by Theorem~\ref{thm:main}.
A rigorous bound relating Binoculars score differences to Kolmogorov complexity differences---as a function of language-model quality---remains an open problem.

\textbf{Limitations.}
Our semantic-load definition (Definition~\ref{def:semantic}) assumes $K(M_i \mid S) = O(1)$, which idealizes the inherently fuzzy notion of natural-language equivalence.
The bounds are asymptotic, and hidden constants may be substantial for finite instances.
A single LLM family handles encoding, decoding, and paraphrasing in our experiments, which may overstate effectiveness against heterogeneous adversaries.
The Binoculars shift, while statistically significant, is moderate in absolute terms (mean $0.310$), and single-threshold detection would have non-trivial error rates.

\textbf{Future work.}
Extensions include formalizing the gap between perplexity-based proxies and $K$-complexity differences, extending the bounds to probabilistic embeddings with approximate semantic preservation, and building multi-feature steganalysis systems that combine perplexity with semantic-coherence metrics.


\end{document}